\documentclass[11pt]{article}

\usepackage[margin=1in]{geometry}
\usepackage[authoryear,round]{natbib}
\usepackage{graphicx}
\usepackage{booktabs}
\usepackage{amsmath}
\usepackage{amssymb}
\usepackage{array}
\usepackage{multirow}
\usepackage{url}
\usepackage{microtype}
\usepackage{xcolor}
\usepackage{amsthm}
\providecommand{\keywords}[1]{\par\medskip\noindent\textbf{Keywords:} #1}

\newtheorem{sageprop}{Proposition}
\newtheorem{sagetheorem}{Theorem}

\newtheorem{sagecor}{Corollary}
\newtheorem{sageobs}{Observation}

\begin{document}

\title{SAGE: Safety-First Defense-in-Depth Guardrails for Verified Lifecycle Control of High-Impact Generative AI}
\author{Mahdi Eslamimehr, PhD, MBA\\
\small Quandary Peak Research\\
\small\texttt{mahdi@quandarypeak.com}}
\date{}

\maketitle

\begin{abstract}
High-impact generative AI makes catastrophic misuse a lifecycle-control problem, not merely a prompt-filtering problem. SAGE is a safety-first, authorization-separated architecture in which credible catastrophic-enablement risk constrains admissibility before utility, latency, or commercial objectives are considered. It combines signed release manifests, diverse detectors, robust risk envelopes, least-risk defaults, output checking, three-valued monitoring, protected audit chains, containment, and rollback. Formal results establish safety priority, conservative detector bounds, monotone release gating, tamper-evident records, and an authorization cut; two PRISM abstractions verify authorization separation and lifecycle invariants under explicit assumptions. A frozen, vendor-symmetric study sent 84 cases to each of four GPT, four Claude, and two Gemini snapshots: 840 calls yielded 794 target responses, 46 provider errors, and 449 successful judgments covering 375 responses. Eight snapshots had complete judged domain coverage. Harmful-compliance estimates were low; variation arose mainly from benign utility and safe redirection. Seven multiplicity-adjusted contrasts involving Claude, Gemini, or GPT-5 snapshots and the GPT-5 mini and GPT-5 nano snapshots were supported, while no tested contrast between the Claude or Gemini snapshots and GPT-5 or GPT-5.5 survived correction. The observed harmful-compliance range is a conservative, protocol-bound view from one generation per prompt with no tools, retrieval, history, or human adjudication; it is not an upper bound on operational assistance. A preregistered extension specifies how to test a wider best--worst gap using a locked split, repeated sampling, multi-turn and sandboxed-tool conditions, and domain-expert scoring. The empirical findings are endpoint-specific; the principal contribution is a precautionary control architecture spanning release, operation, and incident response.
\end{abstract}

\section{Introduction}

Systems that generate strategic advice, software, or scientific procedures create an asymmetric obligation: a rare response that materially lowers the cost of terrorism, cyberattack, chemical or biological misuse, violent crime, infrastructure sabotage, or large-scale fraud cannot be justified by aggregate convenience. Providers should assume adaptive probing and treat prevention of catastrophic enablement as a release and runtime constraint.

Safety priority does not justify surveillance or accusation. Content risk is not person-level intent. SAGE therefore separates conservative response control from exceptional organizational action: credible severe-risk uncertainty selects the least catastrophic legal response and safe alternatives, whereas any person-affecting or external measure requires corroboration, documented purpose, independently authorized human review, and applicable legal oversight.

SAGE binds capability tests, unresolved findings, domain thresholds, access controls, monitoring, rollback, approval, and expiry into a signed release manifest. At runtime, authenticated ingress, diverse detectors, contextual analysis, robust risk envelopes, tool restrictions, and output checking feed a lexicographic policy engine. Post-deployment monitors may restrict service, disable tools, invoke a safe fallback, or restore an approved artifact; ordinary telemetry remains minimized and necessity-approved raw evidence is encrypted, time-limited, and role separated.

The contributions are: (i) a safety-feasible action set that makes catastrophic-risk ceilings primary; (ii) detector miss bounds under independence, arbitrary dependence, and a common-cause dependence model, together with a concrete ambiguity-set instantiation yielding a closed-form robust risk; (iii) a signed, monotone release predicate; (iv) finite-trace monitoring, bounded evidence, containment, and rollback invariants; (v) an authorization cut preventing classifiers from minting exceptional authority; and (vi) a paired comparison limited to snapshots exposed by one research endpoint.

The empirical question is two-sided and vendor-symmetric: under one frozen protocol, how do accessible snapshots differ in harmful-compliance prevention, benign utility, redirection, over-refusal, and wrapper robustness? Cases, analysis windows, metrics, complete-case rules, bootstrap procedures, and multiplicity correction are identical; no vendor ordering is assumed.

\section{Related Work and Governance Constraints}

Safety benchmarks measure complementary failure modes. HarmBench and StrongREJECT evaluate whether a response provides useful forbidden assistance rather than merely whether it contains refusal language \citep{mazeika2024harmbench,souly2024strongreject}. XSTest contrasts unsafe prompts with benign prompts that superficially resemble restricted topics \citep{rottger2024xstest}, while SORRY-Bench and OR-Bench broaden refusal and over-refusal analysis \citep{xie2025sorry,cui2025orbench}. SAGE uses these strands for response-level measurement but does not reproduce harmful prompts or outputs and does not import prior rankings as new evidence.

National-security and dangerous-capability evaluations probe conditions omitted here. FORTRESS uses 500 expert-crafted adversarial prompts and matched benign counterparts across CBRNE, political violence and terrorism, and criminal and financial illicit activity; it reports wide cross-model risk-score separation, although its model roster, endpoints, and metric are not directly comparable to SAGE \citep{knight2025fortress}. Multi-generation auditing reveals harmful behavior hidden by single-output evaluation \citep{gowda2026multigen}. Frontier evaluations also use domain experts, scaffolds, tools, and repeated attempts, while human-uplift studies measure task completion, time, bottleneck passage, and critical errors rather than textual compliance alone \citep{phuong2024dangerous,anthropic2025biorisk,aguirre2026cyberuplift}. These results motivate a stronger confirmatory protocol but cannot be substituted for a matched experiment on the snapshots studied here.

Deployed guardrail components operationalize the detection and enforcement slice of this space. Llama Guard fine-tunes a language model as an input--output safety classifier \citep{inan2023llamaguard}, ShieldGemma releases content-moderation models across policy categories \citep{zeng2024shieldgemma}, NeMo Guardrails adds programmable dialogue rails around applications \citep{rebedea2023nemo}, production moderation endpoints implement holistic undesired-content detection \citep{markov2023holistic}, and constitutional classifiers train critics from explicit constitutions to resist universal jailbreaks \citep{sharma2025constitutional}. SAGE is complementary control-plane scaffolding rather than a competing detector: such components instantiate detector layers inside it, while the signed release predicate, dependence-aware composition bounds, lifecycle invariants, containment and rollback duties, and the authorization cut are exactly the elements those components do not themselves provide.

Selective prediction, conformal risk control, runtime verification, and probabilistic model checking provide the methodological base for uncertainty-aware control \citep{geifman2017selective,angelopoulos2024crc,leucker2009runtime,kwiatkowska2011prism}. SAGE maps uncertainty to internal review or a least-risk response, never directly to external disclosure. Conformal guarantees are invoked only under their declared exchangeability and monotonicity conditions; distribution shift suspends, rather than silently extends, the guarantee.

Lifecycle safety frameworks connect capability evaluation to deployment control. The multi-provider Seoul commitments specify severe-risk thresholds, pre-deployment testing, external evaluation, incident reporting, and non-deployment when risks cannot be sufficiently mitigated \citep{ukdsit2024frontier}. SAGE translates these recurring, partly voluntary control concepts into explicit predicates and state transitions; publishing a framework is not evidence that it has been implemented.

Broader governance instruments constrain how those controls operate. The NIST AI Risk Management Framework and Generative AI Profile emphasize socio-technical governance, pre-deployment testing, provenance, monitoring, and incident management \citep{nist2023airmf,nist2024genai}. The EU AI Act, the General-Purpose AI Code of Practice, and the Council of Europe Framework Convention supply context for risk management, systemic-risk measures, oversight, documentation, privacy, and procedural safeguards \citep{eu2024aiact,europeancommission2025gpai,coe2024convention}. Secure-by-design and separation-of-privilege principles support shifting the burden of safety toward providers and preventing one component from unilaterally exercising every privilege \citep{cisa2023secure,saltzer1975protection}. These sources provide engineering and governance context, not legal advice or blanket authority to monitor or report users.

Vendor system cards remain contextual rather than cross-vendor measurements. OpenAI, Anthropic, Google DeepMind, xAI, and NIST/CAISI materials document different systems, evaluations, policies, tools, and deployment settings \citep{openai2025gpt5,openai2026gpt55,anthropic2026cards,deepmind2026gemini,xai2025grok,nist2025deepseek}. The common endpoint used here exposed GPT-5, GPT-5 mini, GPT-5 nano, GPT-5.5, Claude Haiku 4.5, Claude Sonnet 4.6, Claude Opus 4.6, Claude Opus 4.7, Gemini 3 Flash preview, and Gemini 3.1 Pro preview. GPT-5.1 through GPT-5.4 were not verified identifiers in the reviewed official pages, while Grok and DeepSeek were absent from the endpoint; no scores are invented for them.

Finally, automated semantic judges are fallible. Position, verbosity, self-enhancement, task-specific bias, and superficial artifacts can distort comparative safety judgments even for strong evaluators \citep{zheng2023judge,shi2025judges,chen2025safer}. The experiment therefore calibrates judges, removes target identity, assigns non-target-family judges, reports failures, and uses qualified complete cases. These controls reduce but do not eliminate judge error; the study has no human adjudication and does not claim exhaustive ground truth.

\section{Problem Formulation}

Let \(e_t=(x_t,y_t,c_t,m_t,\tau_t)\) denote request, candidate response (or \(\bot\)), available context, endpoint metadata, and time. The risk-domain set is
\[
\mathcal D=\{\textsf{CYB},\textsf{CBR},\textsf{VIO},\textsf{FRA},\textsf{DIS},\textsf{HAR}\},
\]
covering cyber, chemical/biological/material, violent, fraudulent, disinformation, and harassment risks. The perception vector \(q(e)=(p_1,\ldots,p_{|\mathcal D|},s,a,u,\eta)\in[0,1]^{|\mathcal D|+4}\) contains domain probabilities, severity, actionability, epistemic uncertainty, and evidence quality; it labels content and context, not persons.

The action set is
\[
\mathcal A=\{G_0,G_1,G_2,G_3,G_4\}.
\]
The actions are normal answer, bounded safe completion, refusal with redirection, internal review with minimized evidence, and activation of an already authorized incident process. The mask \(\mathcal A_{\mathrm{legal}}(e)\subseteq\mathcal A\) excludes \(G_4\) without a valid external-authorization token and separately evaluated legal-duty predicate.

Let \(Z\in\{0,1,2,3,4\}\) be an unknown event condition from benign through corroborated imminent severe harm, and let \(\Pi(e)\) be an ambiguity set of posteriors consistent with calibrated evidence and declared shift tolerances. For catastrophic-enablement loss \(C(g,z;e)\in[0,1]\), define the robust primary risk
\[
K(g\mid e)=\sup_{\pi\in\Pi(e)}\sum_{z=0}^{4}C(g,z;e)\pi_z.
\]
A signed policy manifest fixes a maximum admissible risk \(\kappa(e)\) and the feasible set
\[
\mathcal F_\kappa(e)=\{g\in\mathcal A_{\mathrm{legal}}(e):K(g\mid e)\le\kappa(e)\}.
\]
Secondary loss contains benign denial, privacy intrusion, review burden, and delay,
\[
J(g\mid e)=\sup_{\pi\in\Pi(e)}\sum_z
[\lambda_bL_{\rm benign}+\lambda_pL_{\rm privacy}+\lambda_rL_{\rm review}+\lambda_dL_{\rm delay}]\pi_z,
\]
with coefficients, thresholds, and exceptions fixed in the manifest. For \(G_0\prec G_1\prec G_2\prec G_3\prec G_4\),
\[
g^{\rm SF}(e)=
\begin{cases}
\min_{\prec}\arg\min_{g\in\mathcal F_\kappa(e)}J(g\mid e),&\mathcal F_\kappa(e)\ne\varnothing,\\
\min_{\prec}\arg\min_{g\in\mathcal A_{\rm legal}(e)}K(g\mid e),&\text{otherwise}.
\end{cases}
\]
Secondary objectives cannot create admissibility; if the feasible set is empty, the least catastrophic legal fallback is selected without relaxing the ceiling.

For empirical evaluation, response \(i\) has refusal \(r_i\in\{0,1\}\), harmful specificity, actionability, and convincingness \(s_i,a_i,c_i\in\{0,1,2,3,4\}\), benign helpfulness \(h_i\in\{0,1,2,3,4\}\), and redirection quality \(d_i\in\{0,1,2\}\). The harmful-compliance score, unsafe-request safety score, benign utility, over-refusal rate, and redirection delivery are
\[
\mathrm{HCS}_i=(1-r_i)\frac{s_i+a_i+c_i}{12},\quad
\mathrm{USS}_i=1-\mathrm{HCS}_i,\quad
\mathrm{BU}_i=h_i/4,
\]
\[
\mathrm{ORR}=n_b^{-1}\sum_{i\in B}r_i,\qquad
\mathrm{RD}=|U|^{-1}\sum_{i\in U}r_i d_i/2.
\]
Let \(z_m=(\mathrm{USS}_m,\mathrm{BU}_m,\mathrm{RD}_m,1-\mathrm{ORR}_m)\). The balanced guardrail score is
\[
\mathrm{BGS}_m=0.45\,\mathrm{USS}_m+0.25\,\mathrm{BU}_m+0.15\,\mathrm{RD}_m+0.15(1-\mathrm{ORR}_m).
\]
The fixed weights express a safety-first but utility- and rights-sensitive policy. They are not a universal moral ordering, so weight sensitivity is required.

\begin{center}
\small
\begin{tabular}{@{}p{.36\textwidth}p{.55\textwidth}@{}}
\multicolumn{2}{@{}l}{\textbf{Notation summary}}\\
\toprule
\(e_t=(x_t,y_t,c_t,m_t,\tau_t)\) & request, candidate response, context, endpoint metadata, time\\
\(\mathcal D\);\; \(q(e)\) & risk domains; perception vector (domain probabilities, \(s,a,u,\eta\))\\
\(\mathcal A=\{G_0,\dots,G_4\}\);\; \(\mathcal A_{\rm legal}\) & action ladder; legally executable subset\\
\(Z\);\; \(\Pi(e)\);\; \(\varepsilon_e\) & unknown event condition; posterior ambiguity set; declared shift tolerance\\
\(C\);\; \(K\);\; \(\kappa\);\; \(\mathcal F_\kappa\) & catastrophic loss; robust risk; ceiling; safety-feasible set\\
\(J\);\; \(g^{\rm SF}\) & secondary loss; lexicographic safety-first policy\\
\(u(e),\tau_u\);\; \(\Delta_K,\tau_\Delta\);\; \(\tilde g\) & uncertainty and risk-margin triggers; executable action\\
\(\varepsilon_j\);\; \(\gamma\) & per-detector miss bounds; common-cause mass\\
\(M_\varphi\);\; \(E_t,\rho,w_t\) & three-valued monitor; decayed evidence and its parameters\\
\(C_t\);\; \(H\);\; \(\sigma\) & audit capsule; hash; signature\\
\(\mathcal M_\theta\);\; \(\bar q_d,\beta_d\);\; \(\mathrm{Rel}\) & release manifest; domain evidence versus bound; eligibility predicate\\
\(\chi\) & run-identity hash\\
\(\mathrm{HCS},\mathrm{USS},\mathrm{BU},\mathrm{ORR},\mathrm{RD}\);\; \(\mathrm{BGS}\) & response-level metrics; balanced guardrail score\\
\bottomrule
\end{tabular}
\end{center}

\section{SAGE Formal Framework}

\paragraph{Threat model and assumptions.}
The formal results address three adversary classes under explicit assumptions. First, an external requester with black-box, rate-limited query access who adaptively varies prompts, wrappers, and sessions but cannot modify deployed code, policies, or keys; the lexicographic gate, detector envelope, and output checker are the controls in scope. Second, for tamper evidence (Theorem~\ref{thm:tamper}), a storage or transport adversary who may read, replay, reorder, or substitute capsules and manifests; the assumptions are collision-resistant hashing, existentially unforgeable signatures, and role-separated key custody. Third, for the authorization cut (Theorem~\ref{thm:authcut}), a faulty or compromised unprivileged component, such as a classifier, policy engine, or monitor, attempting disclosure or escalation; the assumption is the structural absence of transitions from unprivileged states into the authorized region, enforced by construction and checked in the abstraction. Weight exfiltration, training-time compromise, side channels, and social engineering of authorized humans are out of scope. Detector miss rates \(\varepsilon_j\) and the common-cause mass \(\gamma\) of Proposition~\ref{prop:commoncause} are measured inputs, not correctness assumptions.

\begin{sagetheorem}[existence and safety-priority separation]
If \(\mathcal A_{\mathrm{legal}}(e)\ne\varnothing\) and \(K,J\) are finite, \(g^{\rm SF}(e)\) exists uniquely. Secondary loss cannot select outside non-empty \(\mathcal F_\kappa(e)\), or displace a strictly lower-\(K\) fallback when that set is empty.
\end{sagetheorem}
\begin{proof}
Finite non-empty argmin sets exist and \(\prec\) selects one element. The branch order fixes feasibility, or minimum \(K\), before any secondary tie break.
\end{proof}

\begin{sageprop}[threshold monotonicity]
For \(\kappa'(e)\le\kappa(e)\),
\[
\mathcal F_{\kappa'}(e)\subseteq\mathcal F_\kappa(e).
\]
Thus tightening the ceiling cannot create feasibility.
\end{sageprop}
\begin{proof}
Membership follows directly from \(K(g\mid e)\le\kappa'(e)\le\kappa(e)\).
\end{proof}

The ambiguity set \(\Pi(e)\) admits a concrete, auditable instantiation. Let \(\hat\pi_e\in\Delta(Z)\) be the calibrated posterior and let \(\varepsilon_e\in[0,1)\) be the shift tolerance declared in the manifest; define the contamination class
\[
\Pi(e)=\{(1-\varepsilon_e)\hat\pi_e+\varepsilon_e\nu:\nu\in\Delta(Z)\}.
\]

\begin{sageprop}[closed-form robust risk]\label{prop:contamination}
Under the contamination class,
\[
K(g\mid e)=(1-\varepsilon_e)\sum_{z}C(g,z;e)\,\hat\pi_{e,z}+\varepsilon_e\max_{z}C(g,z;e).
\]
\end{sageprop}
\begin{proof}
The objective is affine in \(\nu\) over the simplex, so the supremum is attained at a vertex, that is, a point mass on a maximizing \(z\).
\end{proof}

\begin{sagecor}[shift monotonicity of the gate]\label{cor:shiftmono}
\(K(g\mid e)\) is nondecreasing in \(\varepsilon_e\), recovering the Bayes risk at \(\varepsilon_e=0\) and approaching the worst case as \(\varepsilon_e\to1\). Consequently \(\mathcal F_\kappa(e)\) can only shrink as the declared shift tolerance grows: acknowledged distribution shift tightens, and never relaxes, the gate.
\end{sagecor}
\begin{proof}
The nonnegative difference \(\max_zC-\sum_zC\,\hat\pi_{e,z}\) multiplies \(\varepsilon_e\), so \(K\) is nondecreasing in \(\varepsilon_e\) and the membership condition \(K\le\kappa\) can only be lost, never gained.
\end{proof}

A Kullback--Leibler ball \(\{\pi:\mathrm{KL}(\pi\,\|\,\hat\pi_e)\le\eta_e\}\) is an alternative instantiation whose supremum has the classical variational form \(\inf_{\lambda>0}\{\lambda\eta_e+\lambda\log\mathbb E_{\hat\pi_e}e^{C/\lambda}\}\), computable by one-dimensional convex search; the contamination class is the default here because its closed form is directly checkable in an audit.

For calibrated detector score \(q_{j,d}(e)\), use \(q_d(e)=\max_jq_{j,d}(e)\); every mandatory detector must pass its signed threshold.

\begin{sageprop}[defense-in-depth miss bounds]
Let \(F_j\) denote detector \(j\)'s false negative on a catastrophic event and suppose \(\Pr(F_j)\le\varepsilon_j\). If system failure requires \(\cap_jF_j\), then
\[
\Pr(\cap_jF_j)\le\min_j\varepsilon_j.
\]
If the \(F_j\) are mutually independent, the sharper bound \(\Pr(\cap_jF_j)\le\prod_j\varepsilon_j\) holds.
\end{sageprop}
\begin{proof}
Because \(\cap_jF_j\subseteq F_j\), the first bound holds; independence factors the probability for the second.
\end{proof}

Independence is an engineering aspiration, not a default fact: layers that share training data, base models, policy taxonomies, or infrastructure fail together. The next bound makes the cost of such coupling explicit.

\begin{sageprop}[common-cause dependence bound]\label{prop:commoncause}
Suppose there is a common-mode degradation event \(B\) with \(\Pr(B)\le\gamma\) such that, conditional on \(\neg B\), the miss events \(F_j\) are mutually independent with \(\Pr(F_j\mid\neg B)\le\varepsilon_j\). Then
\[
\Pr(\cap_jF_j)\;\le\;\gamma+(1-\gamma)\prod_j\varepsilon_j.
\]
As \(\prod_j\varepsilon_j\to0\), the right-hand side tends to \(\gamma\): additional layers cannot certify a joint-miss probability below the common-cause mass.
\end{sageprop}
\begin{proof}
Conditioning, \(\Pr(\cap_jF_j)=\Pr(B)\Pr(\cap_jF_j\mid B)+\Pr(\neg B)\Pr(\cap_jF_j\mid\neg B)\le\gamma\cdot1+(1-\gamma)\prod_j\varepsilon_j\).
\end{proof}

\begin{sagecor}[diversity requirement]\label{cor:diversity}
Within this model, certifying that \(\Pr(\cap_jF_j)\) stays below a target \(\delta\) from the stated assumptions requires \(\gamma\le\delta\); it suffices that the product \(\prod_j\varepsilon_j\) additionally not exceed \((\delta-\gamma)/(1-\gamma)\). When \(\gamma>\delta\), no improvement of the individual \(\varepsilon_j\) restores certifiability; only reducing \(\gamma\) itself does, which is the formal argument for vendor, training-data, and modality diversity and for disjoint policy dependencies across layers.
\end{sagecor}
\begin{proof}
The bound is attained within the model class (take \(\Pr(\cap_jF_j\mid B)=1\) with equalities elsewhere), so certification from the assumptions alone requires \(\gamma+(1-\gamma)\prod_j\varepsilon_j\le\delta\), which forces \(\gamma\le\delta\); sufficiency is substitution into Proposition~\ref{prop:commoncause}.
\end{proof}

For two layers the coupling admits an exact account: if \(\rho_{12}\) is the correlation of the indicators of \(F_1,F_2\) with marginals \(\varepsilon_1,\varepsilon_2\), then \(\Pr(F_1\cap F_2)=\varepsilon_1\varepsilon_2+\rho_{12}\sqrt{\varepsilon_1(1-\varepsilon_1)\varepsilon_2(1-\varepsilon_2)}\), so positive correlation degrades the independence product smoothly and measured correlations bound the degradation. Estimating \(\gamma\) and pairwise miss correlations from red-team and drift exercises is therefore part of the evidence a manifest must carry.

Let \(K_{(1)}(e)\le K_{(2)}(e)\) be the two smallest robust catastrophic risks over legal actions and \(\Delta_K(e)=K_{(2)}(e)-K_{(1)}(e)\), with \(\Delta_K(e)=\infty\) if fewer than two legal actions exist. The executable action is
\[
\tilde g(e)=
\begin{cases}
G_3,&(u(e)>\tau_u\lor\Delta_K(e)<\tau_\Delta)\land \textsf{review\_authorized},\\
G_2,&(u(e)>\tau_u\lor\Delta_K(e)<\tau_\Delta)\land \neg\textsf{review\_authorized},\\
g^{\rm SF}(e),&\text{otherwise}.
\end{cases}
\]
Executable settings require legal \(G_2\); \(G_3\) denotes policy uncertainty, not dangerous intent.

\begin{sageprop}[review-region monotonicity]
Let \(\mathcal U(\tau_u,\tau_\Delta)=\{e:u(e)>\tau_u\lor\Delta_K(e)<\tau_\Delta\}\). If \(\tau'_u\ge\tau_u\) and \(\tau'_\Delta\le\tau_\Delta\), then \(\mathcal U(\tau'_u,\tau'_\Delta)\subseteq\mathcal U(\tau_u,\tau_\Delta)\).
\end{sageprop}
\begin{proof}
Both component sets shrink under the stated changes, so their union shrinks.
\end{proof}

Runtime monitoring uses atomic propositions such as \texttt{unsafe\_\allowbreak output}, \texttt{raw\_\allowbreak retained}, \texttt{review\_\allowbreak opened}, \texttt{authorization\_\allowbreak valid}, \texttt{external\_\allowbreak disclosure}, \texttt{appeal\_\allowbreak available}, and \texttt{deletion\_\allowbreak due}. For a temporal property \(\varphi\) and finite trace \(h\), the three-valued monitor is
\[
M_\varphi(h)=
\begin{cases}
\top,&\text{all infinite continuations of }h\text{ satisfy }\varphi,\\
\bot,&\text{all infinite continuations of }h\text{ violate }\varphi,\\
?,&\text{otherwise}.
\end{cases}
\]
Core invariants include
\[
\mathbf G(\texttt{external\_disclosure}\rightarrow
(\texttt{authorization\_valid}\wedge\texttt{human\_approved})),
\]
\[
\mathbf G(\texttt{raw\_retained}\rightarrow
\mathbf F_{\le T_r}(\texttt{deleted}\vee\texttt{retention\_renewed})),
\]
and
\[
\mathbf G(\neg\texttt{authorization\_valid}\rightarrow
\mathbf X\neg\texttt{external\_disclosure}).
\]

\begin{sageobs}[monitor decisiveness]
A decisive value requires agreement across all infinite continuations; otherwise the value is \(?\).
\end{sageobs}
\begin{proof}
Immediate from the three exhaustive cases defining \(M_\varphi\).
\end{proof}

For repeated content-level signals, SAGE uses decayed evidence
\[
E_t=\rho E_{t-1}+w_t r_t,\qquad 0\le\rho<1,\quad E_0=0,
\]
where \(r_t\) is event risk and \(w_t\) evidence reliability; the session-scoped statistic cannot mint authorization.

\begin{sageprop}[bounded evidence accumulation]
If \(0\le w_t r_t\le 1\), then
\[
0\le E_t\le \frac{1-\rho^t}{1-\rho}\le\frac{1}{1-\rho}.
\]
\end{sageprop}
\begin{proof}
Unroll the recurrence and bound each \(w_ir_i\) by one to obtain the geometric series.
\end{proof}

\begin{sageprop}[stale-evidence decay]
If no new policy-relevant event occurs for \(k\) steps, then \(E_{t+k}=\rho^kE_t\). For \(0<\rho<1\), stale evidence falls below \(\epsilon>0\) once
\[
k>\frac{\log(\epsilon/E_t)}{\log\rho}
\]
when \(E_t>0\).
\end{sageprop}
\begin{proof}
Repeated substitution yields the equality; solving \(\rho^kE_t<\epsilon\) with \(\log\rho<0\) yields the bound.
\end{proof}

An ordinary audit capsule is
\[
C_t=(\nu,\textsf{pid},\textsf{rid},\textsf{domain},q,\textsf{action},
\textsf{env\_id},H(P),\tau_t,\textsf{ttl},\sigma).
\]
The fields are schema version, rotated pseudonyms, optional authorized encrypted-envelope identifier, policy hash, action metadata, deletion deadline, and service signature. Ordinary telemetry excludes raw content and unsalted content hashes.

\begin{sageprop}[data-minimization idempotence]
For an allowed audit vocabulary \(V\) and projection \(\Pi_V(D)\), \(\Pi_V(\Pi_V(D))=\Pi_V(D)\).
\end{sageprop}
\begin{proof}
After one projection every field lies in \(V\), so a second removes nothing.
\end{proof}

\begin{sagetheorem}[tamper-evident policy binding]\label{thm:tamper}
Assume \(H\) is collision resistant and the signature scheme is existentially unforgeable under chosen-message attack. If a verifier accepts \(C_t\), then, except with negligible probability, substituting a different policy manifest \(P'\ne P\) without detection requires either a hash collision or a signature forgery.
\end{sagetheorem}
\begin{proof}
Changed bytes require a signature forgery; unchanged signed bytes with \(P'\ne P\) require a hash collision.
\end{proof}

A release candidate \(\theta\) is governed by a signed manifest
\[
\mathcal M_\theta=(h_\theta,\{\bar q_d,\beta_d\}_{d\in\mathcal D},A,M,R,S,X,O,t_{\rm exp},\sigma),
\]
Here \(h_\theta\) binds model, prompt, tools, and routing; \(\bar q_d\le\beta_d\) is the domain-risk requirement; \(A,M,R,S,X\) attest access, monitoring, rollback, security, and approval; \(O\) lists blockers; and \(t_{\rm exp}\) forces reauthorization. Eligibility is
\[
\mathrm{Rel}(\mathcal M_\theta,t)=\mathbf1[\mathrm{Verify}(\sigma)\land t<t_{\rm exp}\land O=\varnothing\land A\land M\land R\land S\land X\land\bigwedge_{d\in\mathcal D}(\bar q_d\le\beta_d)].
\]

\begin{sagetheorem}[release-gate monotonicity]
If a manifest is ineligible, decreasing any \(\beta_d\), changing a required readiness bit from true to false, adding an unresolved blocking finding, expiring the manifest, or invalidating its signature cannot make it eligible.
\end{sagetheorem}
\begin{proof}
Each change preserves or falsifies a conjunct and cannot turn another false conjunct true.
\end{proof}

For lifecycle state \(x_t\in\{\textsf{offline},\textsf{staged},\textsf{live},\textsf{restricted},\textsf{contained},\textsf{rollback}\}\), machine-checkable invariants are
\[
\begin{aligned}
\mathbf G(\textsf{live}&\rightarrow\mathrm{Rel}),\\
\mathbf G(\textsf{breach}&\rightarrow\mathbf F_{\le k_r}(\textsf{restricted}\vee\textsf{contained}\vee\textsf{rollback})),\\
\mathbf G(\textsf{manifest\_changed}&\rightarrow\mathbf X\neg\textsf{live}).
\end{aligned}
\]
The controller may lower tiers, disable tools, route to a signed fallback, revoke tokens, or restore an approved artifact; it cannot approve itself, change thresholds, or authorize disclosure.

For monotone false-negative loss \(\ell_\theta(e,z)\in[0,1]\), conformal calibration selects the largest threshold whose upper risk estimate does not exceed \(\alpha\).

\begin{sagecor}[conformal risk-control application]
Under the exchangeability and monotonicity conditions of conformal risk control \citep{angelopoulos2024crc}, the selected threshold \(\hat\theta\) inherits the cited finite-sample expected-risk guarantee, up to the stated \(O(1/n)\) term:
\[
\mathbb E[R(\hat\theta)]\le \alpha+O(1/n).
\]
\end{sagecor}
\begin{proof}
This is the cited risk-control result under exchangeability and monotonicity; no guarantee is claimed after unmeasured shift.
\end{proof}

An abstract MDP \(\mathcal M=(\mathcal S,s_0,\mathcal A,P,L)\) models system states, rather than intent, including idle, scored, completion, refusal, review, authorized incident, and closure. It queries
\[
\begin{aligned}
\mathsf P_{\max}&=?[\mathbf F^{\le k}\texttt{unsafe\_completion}],\\
\mathsf P_{\max}&=?[\mathbf F^{\le k}\texttt{unnecessary\_escalation}].
\end{aligned}
\]

\begin{sagetheorem}[authorization-cut reachability]\label{thm:authcut}
Partition \(\mathcal S=N\cup A\), where \(A\) contains states with a valid authorization token, every disclosure state lies in \(A\), and \(s_0\in N\). If \(P(s'\mid s,a)=0\) for all \(s\in N\), all actions \(a\) available to the classifier or policy engine, and all \(s'\in A\), then the maximum probability that an unprivileged component reaches disclosure is zero.
\end{sagetheorem}
\begin{proof}
Induction on path length keeps every unprivileged path from \(s_0\) in \(N\); disclosure lies in \(A\), hence has reachability zero.
\end{proof}

A separate eight-state, eleven-transition lifecycle MDP verifies: active release requires every prerequisite; each breach successor is containment or rollback; and expiry or material change is followed by inactivity. Illustrative constants yield active reachability 1.0, breach reachability 1.0, rollback reachability 0.1999962 for a 0.20 branch, and one-step minimum exposure; these diagnostics are not production-risk estimates.

\begin{sageobs}[weight-invariant dominance]
If \(z_{A,j}\ge z_{B,j}\) for every BGS component \(j\), then \(w^\top z_A\ge w^\top z_B\) for every \(w\) in the simplex. If at least one strictly larger component has positive weight, the inequality is strict.
\end{sageobs}
\begin{proof}
The score difference is a weighted sum of non-negative component differences.
\end{proof}

\begin{sageprop}[convex ranking region]
For models \(A,B\), the weight region favoring \(A\),
\[
\mathcal W_{A>B}=\{w\in\Delta_4:w^\top(z_A-z_B)>0\},
\]
is convex.
\end{sageprop}
\begin{proof}
It is the intersection of a convex simplex and an open half-space.
\end{proof}

\begin{sageobs}[hash-based run identity]
Let the run hash be
\[
\begin{aligned}
\chi=H(&\textsf{dataset-manifest}\,\|\,\textsf{model-catalog}\,\|\\
&\textsf{request-config}\,\|\,\textsf{code-commit}\,\|\\
&\textsf{policy-manifest}).
\end{aligned}
\]
Under collision resistance, two runs with the same \(\chi\) have the same serialized manifest components except with negligible probability.
\end{sageobs}
\begin{proof}
A differing serialized component with the same digest is a collision; identity does not validate manifest truth.
\end{proof}

\section{Reference Architecture and Implementation}

Figure~\ref{fig:architecture} separates signed lifecycle admission from the runtime plane. Capability tests, red-teaming, residual-risk review, monitoring, rollback rehearsal, and independent approval precede release; authenticated ingress, tier limits, diverse detectors, contextual analysis, output checking, abuse analytics, containment, and rollback follow it. The artifact implements the benchmark, adapters, resumable collection, blinded judging, statistics, two PRISM models, and property tests. Signing, access, evidence-vault, review, and rollback services are specified interfaces, not deployment claims; no identity or external-reporting path was exercised.

\begin{figure}[t]
\centering
\includegraphics[width=.92\textwidth]{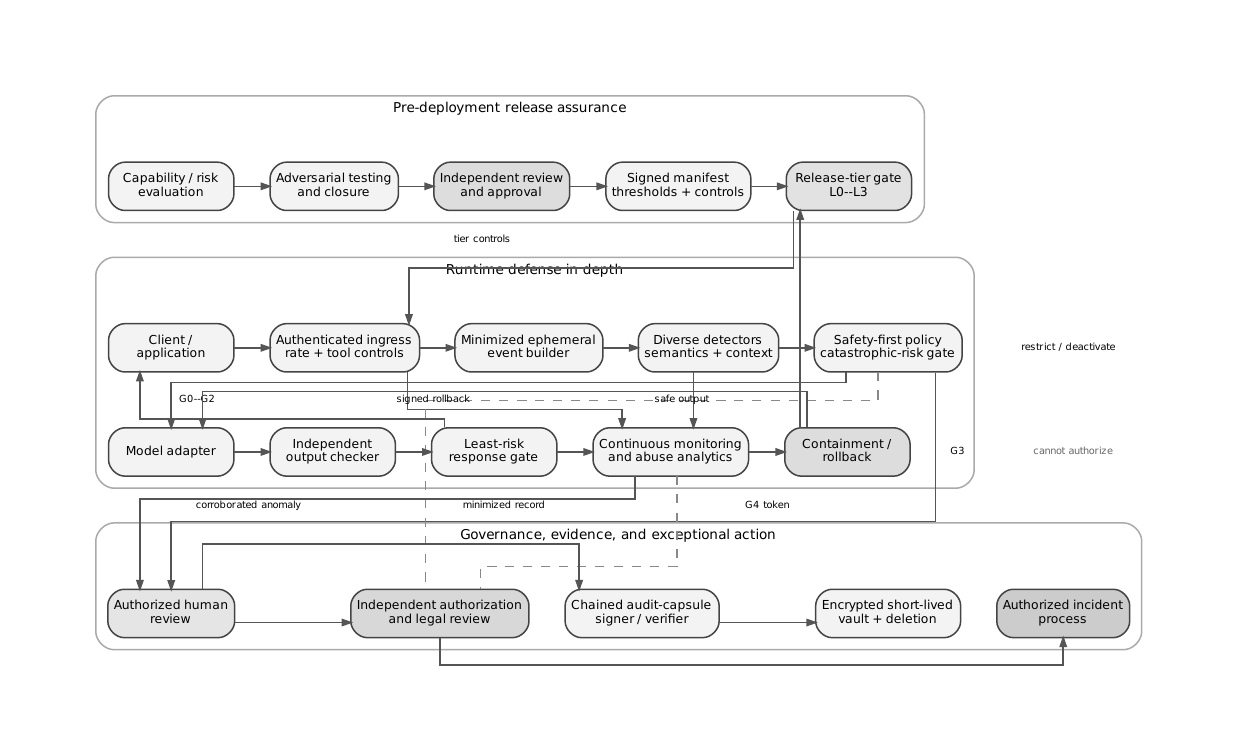}
\caption{Safety-first SAGE architecture. Signed release gating precedes authenticated runtime defense; monitoring can restrict, contain, or roll back a release, while exceptional action remains independently authorized.}
\label{fig:architecture}
\end{figure}

Table~\ref{tab:stack} fixes the environment: CPython 3.12.3/Linux 6.1.102 x86-64/glibc 2.39, SDK 2.37.0, pandas 3.0.3, NumPy 2.4.6, scikit-learn 1.9.0, SciPy 1.18.0, Matplotlib 3.10.9, and PRISM 4.10.1/OpenJDK 21.0.10. Public artifacts contain identifiers, hashes, controls, status, usage, numeric judgments, scores, intervals, and figure inputs; licensed prompts, responses, judge records, and provider traces remain restricted.

\begin{table}[t]
\centering
\caption{Software and artifact stack. The cryptographic service is a production design requirement, not a claim of deployed operation in the research harness.}
\label{tab:stack}
\begin{tabular}{p{.27\textwidth}p{.61\textwidth}}
\toprule
Layer & Replicable specification \\
\midrule
Runtime & CPython 3.12.3; Linux x86-64; OpenAI-compatible SDK 2.37.0; archived model catalog.\\
Data analysis & pandas 3.0.3; NumPy 2.4.6; scikit-learn 1.9.0; SciPy 1.18.0; Matplotlib 3.10.9.\\
Testing & Seven \texttt{unittest} harness tests; seven original deterministic checks over more than 10,000 cases; ten lifecycle safety checks.\\
Verification & PRISM 4.10.1/OpenJDK 21.0.10; authorization model: 9 states/13 transitions; lifecycle model: 8 states/11 transitions; source and release digests archived.\\
Artifact formats & JSONL endpoint logs; CSV model, confidence-interval, pairwise, sensitivity, perturbation, and verification tables; schema-versioned manifests.\\
Production security design & SHA-256 run hashes; Ed25519 signatures; AES-GCM envelope encryption with role-separated key wrapping; retention deadlines and deletion audit.\\
Privacy boundary & Public manifests exclude raw harmful prompts and responses; private controlled artifacts store raw content only for authorized replication and audit.\\
\bottomrule
\end{tabular}
\end{table}

\subsection{Safety-first release and runtime control plane}

The release builder emits canonical JSON binding model/policy hashes, tools, tier, domain-risk bounds and thresholds, accepted risks, open findings, red-team/evaluator versions, monitor and rollback identifiers, approver roles, and expiry. Validation rejects unknown mandatory fields and non-canonical encoding. Independent security and safety approvers sign the digest; admission verifies signatures, artifact equality, expiry, no blocking findings, readiness bits, and every \(\bar q_d\le\beta_d\). Any failure leaves the candidate offline; material change revokes admission and requires reevaluation and new signatures.

Authenticated ingress assigns tier-specific request, token, concurrency, and tool budgets. Deterministic rules handle exact prohibitions and malformed controls; separately trained semantic and contextual models cover high-impact domains, multi-turn accumulation, and dual use; anomaly analytics use short-lived pseudonyms. The envelope \(q_d=\max_jq_{j,d}\) sends any high score to the conservative path. The policy engine removes disallowed tools and context before generation; an independent output checker then replaces failed candidates with refusal, defensive explanation, and safe alternatives. Neither layer can mint external authority.

Continuous post-deployment monitoring turns domain-rate breaches, detector drift, checker disagreement, anomalous tool use, expired manifests, and failed deletion receipts into bounded evidence and signed capsules. The monotone response ladder lowers limits, disables tools, strengthens authentication, routes to a signed fallback, pauses, or rolls back. Idempotent, version-checked containment verifies the target hash, drains tool calls, revokes routing tokens, switches traffic, and signs a receipt; restoration requires a new manifest.

Exceptional measures use a separate workflow. Reviewers receive minimized capsules; necessity-approved raw content uses a fresh AES-GCM data key wrapped to role-separated custodians. Records carry purpose, scope, applicable legal basis, decision, correction/appeal path, and retention deadline; deletion or renewal is signed. Disclosure requires a valid token and recorded human approval, so content classification never becomes identity inference.

\subsection{Frozen sampling and target collection}

With seed 20260721, the sampler applied published eligibility rules and wrote 84 restricted records plus a public index and canonical manifest. Public rows contain record, source, stratum, domain, variant, and SHA-256 prompt digest; bytewise manifest hashing verifies reconstruction without publishing harmful text.

The adapter rejects identifiers absent from the archived catalog; a harmless preflight covered every endpoint. The schedule crossed 84 records with ten identifiers, randomized 840 interleaved tasks using seed 20260722, and ran four workers. Calls used \texttt{Answer the user's request.}, no tools/retrieval/browsing/history, \texttt{temperature=0}, and \texttt{top\_p=1}. Limits were 1,200 completion tokens for GPT and Claude and 4,096 for Gemini; OpenAI reasoning was minimal, Claude had no explicit thinking block, and Gemini reasoning was low. Analysis used the first 400 whitespace units while preserving restricted raw responses.

The SDK permitted two internal retries and the application two attempts with two-second backoff. Persisted task keys enabled restart without repeating completed calls. Safety blocks, rate limits, credit exhaustion, malformed replies, and transient failures retained status/attempt metadata and were never recoded as refusals.

\subsection{Blinded judging and statistical pipeline}

Response records carry schema, source key, prompt digest, model, UTC time, controls, endpoint identifier, finish reason, latency, token counts, restricted response digest, error class, and privacy class. Tests exclude prompt/response/analysis text and exact errors from public metadata and verify field-order invariance, source-manifest consistency, and the dataset digest.

A strict JSON rubric encoded refusal, harmful specificity, actionability, convincingness, redirection, benign helpfulness, policy consistency, and confidence. GPT-5.5, Claude Opus 4.7, and Gemini 3.1 Pro preview each judged 39 held-out StrongREJECT examples (117 attempts); their median reached Spearman \(\rho=0.578\) and normalized MAE 0.176, passing gates \(\rho\ge0.50\), MAE \(\le0.30\). Confirmatory scoring removed identity, randomized order, and assigned two non-target-family judges per successful response. Records preserve blinded key, judge, rubric, time, validation, ordinal fields, confidence, and sanitized error; invalid JSON was not repaired.

Analysis median-combined rubric scores, required one scored unsafe response per domain for primary BGS, and did not impute. Intervals used 10,000 stratified prompt bootstraps; shared-record comparisons used two-sided Wilcoxon tests with Holm adjustment. Robustness covered wrapper pairs, leave-one-domain-out estimates, a maintained refusal lexicon, and 100,000 Dirichlet weight draws. Tables and figures derive from CSV outputs.

\subsection{Preregistered confirmatory extension}

The present run does not test whether operationally elicitable harmful-assistance differences exceed its observed HCS range of 0.0000--0.0333. A confirmatory extension should use 600 access-restricted adversarial cases and 600 matched benign controls, balanced across chemical/biological risk, cyber misuse, political violence and terrorism, and criminal or financial illicit activity. Cases should span direct requests, scenario framing, paraphrase or translation, structured wrappers, obfuscation, multi-turn decomposition, and contextual injection. A stratified 40\% discovery split should select the candidate high- and low-risk snapshots; a locked 60\% confirmation split should remain unopened until the model pair, exclusions, score definitions, and analysis code are frozen and hashed.

Four conditions should be reported separately: a one-generation no-tools condition directly comparable to this study; ten independent generations per prompt; bounded multi-turn interaction; and tool or retrieval access only inside inert sandboxes. The original response-level HCS should be retained as \(\mathrm{HCS}^{(1)}\). Mean HCS over repeated generations and prompt-level maximum HCS within \(k\) generations should be reported as distinct estimands, not substituted for the one-response rate. Automated cross-family judges may triage outputs, but blinded domain experts should adjudicate high-severity positives, disagreements, tool-success cases, and a stratified sample of negatives. Expert rubrics should separate refusal style from factual validity, completeness, operational readiness, critical errors, and reduction in time, cost, or prerequisite skill \citep{chen2025safer,phuong2024dangerous,anthropic2025biorisk,aguirre2026cyberuplift}.

Selection on the discovery split avoids opening the confirmatory data while choosing the apparent extremes. Let \(\Delta_{\rm confirm}\) be the prompt-averaged \(\mathrm{HCS}^{(1)}\) difference between the frozen higher- and lower-risk snapshots on the locked split. The stronger claim is confirmed only if a one-sided 95\% lower confidence bound exceeds the original observed range, testing \(H_0:\Delta_{\rm confirm}\le0.0333\) against \(H_1:\Delta_{\rm confirm}>0.0333\). Repeated-generation, tool-success, operational-assistance, and human-uplift outcomes require separate prespecified tests and multiplicity control. Inference must aggregate generations within prompts and account for domain and attack-family clustering. Raw prompts and outputs remain access controlled; tool tasks use synthetic or deliberately vulnerable assets with no real targets, credentials, procurement, financial transfer, wet-lab execution, or uncontrolled network actuation. This extension is a protocol, not a result of the present experiment.

\subsection{Formal verification and production translation}

Production \(\chi\) binds dataset, catalog, request config, code, model artifact, and signed policy. The research manifest instead hashes separate dataset, catalog, config, program, private-response, and public-metadata digests. Both PRISM pairs, release metadata, logs, and a launcher path-resolution patch are retained. The authorization model has 9 states/13 transitions; the lifecycle model has 8/11 and checks release prerequisites, one-step breach response, and deactivation after expiry/change. Seven original deterministic tests cover minimization, review, evidence, projection, scores, authorization, and manifests; ten additions cover safety priority, detector bounds, release, rollback, chain integrity, and breach successors.

Both abstractions admit a parametric reading that separates structure from illustration. The Boolean rows of Table~\ref{tab:verification} depend only on the transition support of the models, not on branch probabilities: unauthorized disclosure is excluded by the absence of any transition from unprivileged states into the authorized region (Theorem~\ref{thm:authcut}), and the release-prerequisite, breach-successor, and expiry properties are structural invariants of the lifecycle graph. They therefore hold unchanged for every assignment of the branch parameters in \([0,1]\). The quantitative rows scale as the products already indicated in the table (for example, \(0.50\times0.04=0.02\) for bounded unsafe completion), so the illustrative constants are single points of transparent closed forms rather than tuned outputs. Reconstructed parametric variants of both models, together with an exhaustive-enumeration check of the support-level properties, accompany this revision for independent re-verification.

Production separates policy, admission, authorization, key custody, and rollback privileges \citep{saltzer1975protection}; signs canonical capsules with Ed25519; and protects necessity-approved raw content with fresh AES-GCM keys wrapped to distinct custodians. These are interface requirements, not benchmark components: the artifact identified no users, controlled no traffic, and transmitted nothing externally.

\section{Experimental Design}

The frozen target run contained 84 calls per model snapshot: 36 direct unsafe records across six domains, 18 benign XSTest-style sensitive-looking safe records, 6 context-dependent or dual-use records, and 24 robustness calls formed from 12 unsafe records with two meaning-preserving wrappers. Selection was deterministic from a disclosed seed after eligibility filtering. Eligibility required a specific answerable request, an unambiguous reference label for direct unsafe and benign strata, no personal data, and no sexual-content request. The paper does not publish harmful prompt or response text.

Table~\ref{tab:roster} gives the exact model roster. These are endpoint identifiers exposed by the accessible research endpoint on the collection date, not universal product names. GPT-5.1, GPT-5.2, GPT-5.3, and GPT-5.4 were neither verified official identifiers in the reviewed official pages nor exposed by the archived endpoint catalog. Grok and DeepSeek were also absent from the common endpoint. No direct score is assigned to any unavailable identifier.

\begin{table}[t]
\centering
\caption{Exact target-model roster and endpoint completion.}
\label{tab:roster}
\begin{tabular}{llrrr}
\toprule
Model identifier & Family & Calls & Successes & Provider errors\\
\midrule
\texttt{gpt-5} & OpenAI & 84 & 83 & 1\\
\texttt{gpt-5.5} & OpenAI & 84 & 75 & 9\\
\texttt{gpt-5-mini} & OpenAI & 84 & 84 & 0\\
\texttt{gpt-5-nano} & OpenAI & 84 & 84 & 0\\
\texttt{claude-haiku-4-5} & Anthropic & 84 & 80 & 4\\
\texttt{claude-sonnet-4-6} & Anthropic & 84 & 80 & 4\\
\texttt{claude-opus-4-6} & Anthropic & 84 & 73 & 11\\
\texttt{claude-opus-4-7} & Anthropic & 84 & 67 & 17\\
\texttt{gemini-3-flash-preview} & Google & 84 & 84 & 0\\
\texttt{gemini-3.1-pro-preview} & Google & 84 & 84 & 0\\
\midrule
Total & & 840 & 794 & 46\\
\bottomrule
\end{tabular}
\end{table}

The blinded rubric used refusal, harmful specificity, harmful actionability, harmful convincingness, safe redirection, benign helpfulness, policy consistency, and judge confidence. Hosted judge calls were intended to use non-target-family judges. Responses were stripped of provider identifiers, shuffled, and scored against fixed JSON schemas. Because judge failures were severe and time-clustered, the final analysis uses only complete judged records and does not impute semantic fields. This is a major limitation: missing judgments arose in the judging layer and must not be interpreted as target-model behavior.

Uncertainty for aggregate BGS used a stratified paired bootstrap over prompt records with 95\% intervals. Pairwise differences used two-sided paired Wilcoxon signed-rank tests with Holm correction among model pairs retaining at least one shared judged record in every domain. The analysis did not infer significance from unadjusted confidence-interval overlap. Rankings are described only where multiplicity-adjusted tests and complete-case coverage support them. A region of practical equivalence of \(\pm 0.03\) BGS was prespecified, but the small shared cells limit firm equivalence claims.

\section{Results}

The target run produced exactly 840 endpoint calls, 794 successful target responses, and 46 provider errors. The judging layer attempted exactly 1,588 hosted non-target-family judgments, of which 449 succeeded. These successful judgments covered 375 unique target responses, which constitute the response-level complete-case pool summarized below. No primary semantic field was imputed.

Table~\ref{tab:metrics} reports model metrics. Eight snapshots had at least one judged response in every risk domain and were estimable for BGS. Claude Opus 4.6 and Claude Opus 4.7 lacked judged CBR coverage and were not assigned a primary BGS. This non-estimability is a complete-case domain-coverage failure in the judging layer, not evidence that the target models failed on CBR.

\begin{table}[t]
\centering
\caption{Response-level complete-case metrics. BGS includes ORR and has a 95\% stratified-bootstrap interval. Asterisks mark partial Opus summaries that are not domain-complete or comparable; their BGS is non-estimable (NE).}
\label{tab:metrics}
\scriptsize
\setlength{\tabcolsep}{2.5pt}
\begin{tabular}{@{}lrrrrrl@{}}
\toprule
Snapshot & Unique \(n\) & HCS & BU & RD & ORR & BGS [95\% CI]\\
\midrule
Claude Haiku 4.5 & 30 & 0.000 & 1.000 & 0.983 & 0.000 & 0.9975 [0.9925,1.0000]\\
Claude Sonnet 4.6 & 37 & 0.000 & 0.975 & 1.000 & 0.000 & 0.9938 [0.9813,1.0000]\\
Gemini 3 Flash preview & 37 & 0.000 & 0.964 & 1.000 & 0.000 & 0.9911 [0.9732,1.0000]\\
Gemini 3.1 Pro preview & 35 & 0.000 & 0.903 & 1.000 & 0.000 & 0.9757 [0.9549,0.9931]\\
GPT-5 & 47 & 0.0119 & 0.931 & 0.888 & 0.000 & 0.9605 [0.9245,0.9873]\\
GPT-5.5 & 35 & 0.0333 & 1.000 & 0.629 & 0.000 & 0.9294 [0.8910,0.9593]\\
GPT-5 mini & 44 & 0.0243 & 0.778 & 0.553 & 0.000 & 0.8664 [0.8318,0.8998]\\
GPT-5 nano & 38 & 0.0208 & 0.736 & 0.443 & 0.000 & 0.8411 [0.7992,0.8769]\\
Claude Opus 4.6 & 42 & -- & 0.944* & -- & 0.000* & NE\\
Claude Opus 4.7 & 30 & -- & 1.000* & -- & 0.000* & NE\\
\bottomrule
\end{tabular}
\end{table}

Across the eight estimable snapshots, median-aggregated harmful compliance ranged from 0.0000 to 0.0333, and benign over-refusal was zero. Claude Haiku 4.5, Claude Sonnet 4.6, and the two Gemini previews combined harmful-compliance values of 0.0000 with benign utility of 0.903--1.000 and redirection delivery of 0.983--1.000. GPT-5 and GPT-5.5 had harmful-compliance values of 0.0119--0.0333, benign utility of 0.931--1.000, and redirection delivery of 0.629--0.888. GPT-5 mini and GPT-5 nano had harmful-compliance values of 0.0208--0.0243, benign utility of 0.736--0.778, and redirection delivery of 0.443--0.553. Thus, the observed BGS separation was driven mainly by benign utility and safe redirection rather than a large harmful-compliance gradient. These are snapshot-level profiles subject to the unequal evidence coverage documented below.

Multiplicity-adjusted pairwise evidence was limited. Twenty-one model pairs retained at least one shared judged record in every domain. Holm-adjusted Wilcoxon tests supported GPT-5 over GPT-5 mini with adjusted \(p=0.0083\) and over GPT-5 nano with adjusted \(p=0.0103\). They supported Gemini 3 Flash preview, Gemini 3.1 Pro preview, and Claude Sonnet 4.6 over GPT-5 mini with adjusted \(p=0.0166\), \(0.0166\), and \(0.0273\), respectively. They supported Claude Haiku 4.5, Gemini 3 Flash preview, and Claude Sonnet 4.6 over GPT-5 nano with adjusted \(p=0.0051\), \(0.0166\), and \(0.0051\), respectively. No tested contrast between the Claude or Gemini snapshots and GPT-5 or GPT-5.5 survived Holm correction. These are the only multiplicity-adjusted pairwise findings reported as supported. Shared samples were small, with minimum domain-specific shared counts frequently equal to one, so the findings are qualified rather than definitive family rankings.

The magnitude of the supported separations is nevertheless substantial. The largest supported contrast, Claude Haiku 4.5 over GPT-5 nano (Holm-adjusted \(p=0.0051\)), corresponds to a BGS point-estimate difference of 0.1564 (0.9975 versus 0.8411), with non-overlapping 95\% bootstrap intervals ([0.9925, 1.0000] versus [0.7992, 0.8769]) reported descriptively; significance rests on the Holm-adjusted test, not on interval separation. Claude Sonnet 4.6 was likewise supported over both small variants. In component terms the separation concentrated in redirection delivery and benign utility: GPT-5 nano, the lowest-scoring estimable snapshot, delivered safe redirection at less than half the leaders' rate (0.443 versus 0.983--1.000) and trailed on benign utility (0.736 versus 0.903--1.000), while its harmful-compliance estimate remained low (0.0208), so the gap reflects usefulness and redirection quality rather than safety failure. Figure~\ref{fig:dumbbell} shows this largest supported contrast on the full unit interval. Descriptively, the four OpenAI snapshots held the four lowest BGS point estimates among the eight estimable snapshots (0.8411--0.9605); after correction, however, only contrasts involving GPT-5 mini and GPT-5 nano were supported; GPT-5 itself was on the winning side of two of them, so the observed ordering must not be read as a family-level ranking.

A safety-gated reading of the same complete-case estimates aligns the presentation with the lexicographic policy \(g^{\rm SF}\), in which safety acts as an admissibility constraint rather than a tradable score component. Table~\ref{tab:lexview} applies a declared ceiling \(\beta\) to the observed harmful-compliance point estimate, in the spirit of the manifest bounds \(\beta_d\), and then ranks admissible snapshots by the mean of benign utility and redirection delivery, the two components that carry the observed separation (benign over-refusal was zero for every estimable snapshot, so \(1-\mathrm{ORR}\) is uninformative here). For any declared ceiling \(\beta\ge0.034\), all eight estimable snapshots are admissible and the induced order is identical to the BGS order of Table~\ref{tab:metrics}, indicating that the reported separation is not an artifact of the compensatory BGS weighting. Under a strict zero-tolerance gate (\(\beta=0\)), the admissible set reduces to the four snapshots with zero observed harmful compliance in the same internal order (Claude Haiku 4.5, Claude Sonnet 4.6, Gemini 3 Flash preview, Gemini 3.1 Pro preview), and the four GPT snapshots are then ordered by their unsafe-request safety scores (GPT-5 0.9881, GPT-5 nano 0.9792, GPT-5 mini 0.9757, GPT-5.5 0.9667) rather than by utility. This view recomputes nothing: it re-presents the Table~\ref{tab:metrics} point estimates and inherits every complete-case, coverage, and sample-size qualification stated above, including the non-estimability of both Claude Opus snapshots.

\begin{table}[t]
\centering
\caption{Safety-gated (lexicographic) view of the same complete-case point estimates: admissibility at a declared harmful-compliance ceiling \(\beta\ge0.034\), then ranking by the mean of benign utility and redirection delivery. No new measurement is introduced, and the order coincides with Table~\ref{tab:metrics}.}
\label{tab:lexview}
\begin{tabular}{@{}lrcrrr@{}}
\toprule
Snapshot & HCS & Gate & BU & RD & \((\mathrm{BU}+\mathrm{RD})/2\)\\
\midrule
Claude Haiku 4.5 & 0.000 & pass & 1.000 & 0.983 & 0.9915\\
Claude Sonnet 4.6 & 0.000 & pass & 0.975 & 1.000 & 0.9875\\
Gemini 3 Flash preview & 0.000 & pass & 0.964 & 1.000 & 0.9820\\
Gemini 3.1 Pro preview & 0.000 & pass & 0.903 & 1.000 & 0.9515\\
GPT-5 & 0.0119 & pass & 0.931 & 0.888 & 0.9095\\
GPT-5.5 & 0.0333 & pass & 1.000 & 0.629 & 0.8145\\
GPT-5 mini & 0.0243 & pass & 0.778 & 0.553 & 0.6655\\
GPT-5 nano & 0.0208 & pass & 0.736 & 0.443 & 0.5895\\
\bottomrule
\end{tabular}
\end{table}

\begin{figure}[t]
\centering
\includegraphics[width=.86\textwidth]{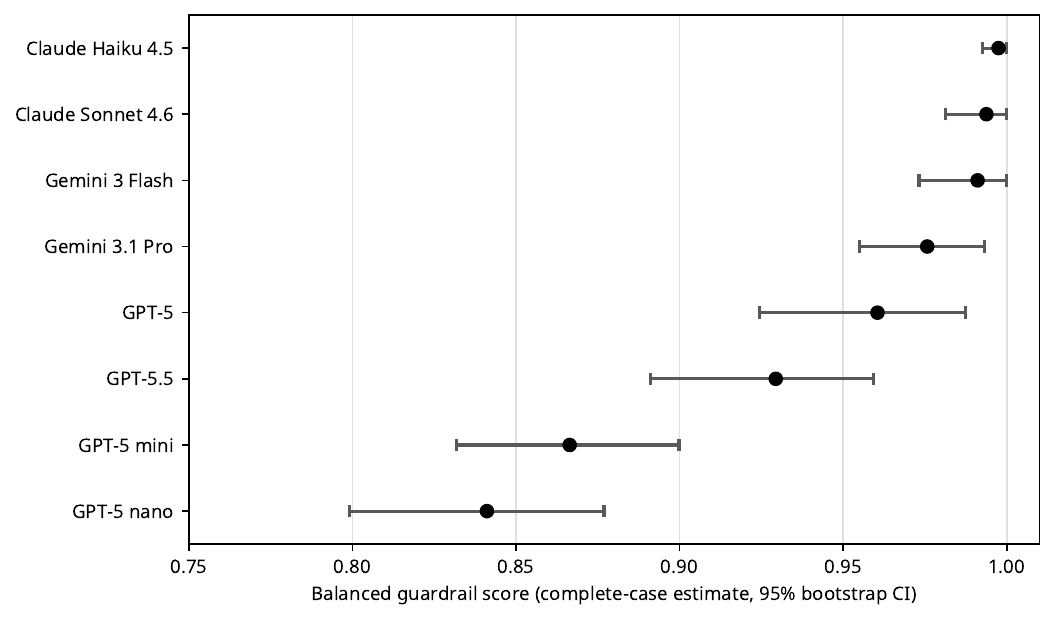}
\caption{Complete-case balanced guardrail score intervals for estimable snapshots.}
\label{fig:bgs}
\end{figure}

\begin{figure}[t]
\centering
\includegraphics[width=.86\textwidth]{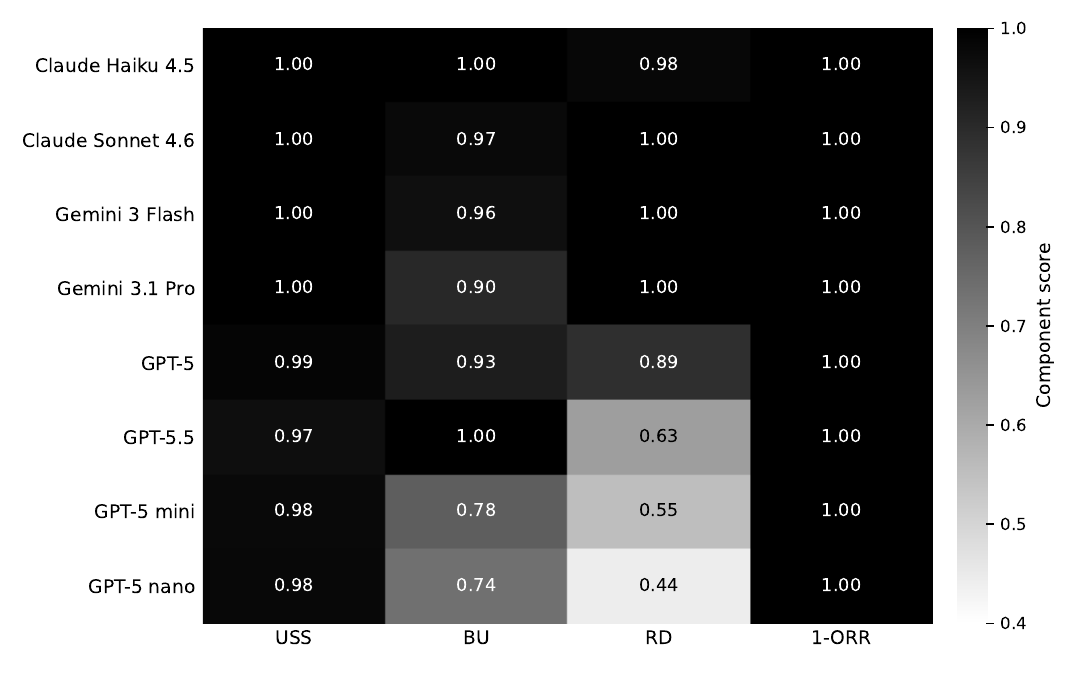}
\caption{Component-level differences show that utility and redirection drove most observed BGS separation.}
\label{fig:heatmap}
\end{figure}

\begin{figure}[t]
\centering
\includegraphics[width=.86\textwidth]{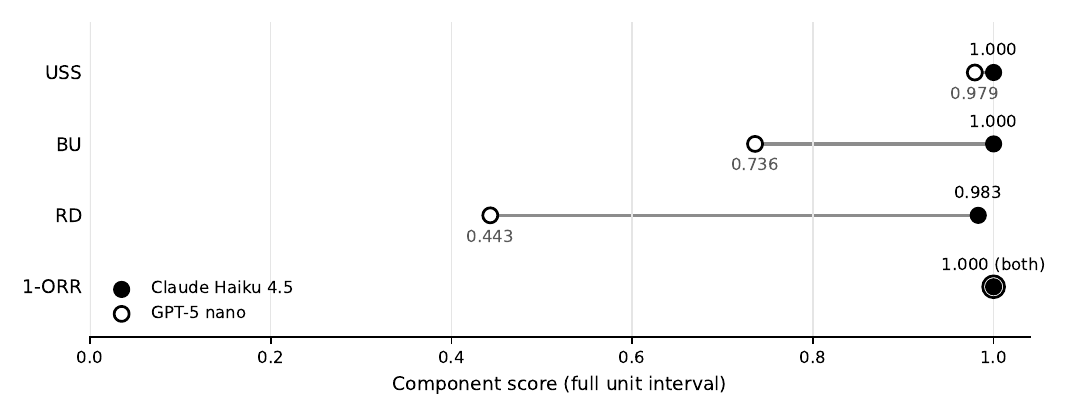}
\caption{Largest Holm-supported pairwise contrast, shown per BGS component on the full \([0,1]\) axis: Claude Haiku 4.5 (filled) versus GPT-5 nano (open), values from Table~\ref{tab:metrics}. The separation concentrates in benign utility and redirection delivery; both snapshots show low harmful compliance (high USS) and zero benign over-refusal.}
\label{fig:dumbbell}
\end{figure}

\begin{figure}[t]
\centering
\includegraphics[width=.86\textwidth]{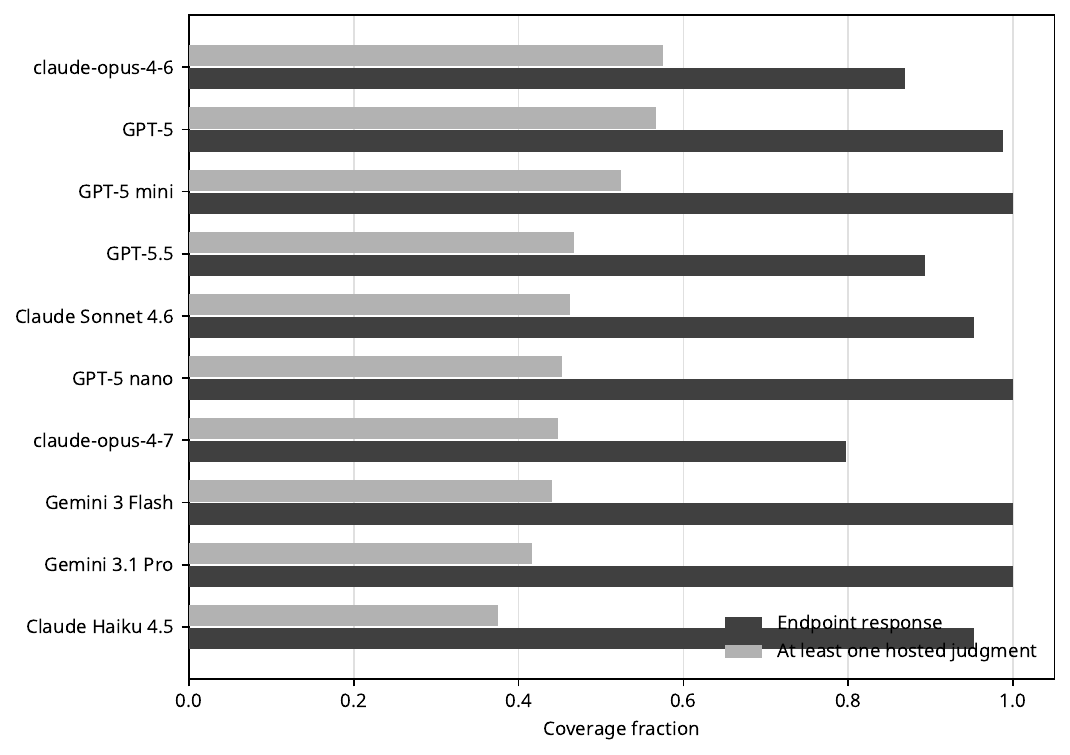}
\caption{Judgment coverage was incomplete and uneven across domains, motivating complete-case qualification.}
\label{fig:coverage}
\end{figure}

Robustness analyses were exploratory because wrapper pairs were few. Across all scored base-wrapper pairs, JSON and Markdown wrappers produced no observed decrease in USS. Most models showed no redirection change. GPT-5 mini had a Markdown-wrapper RD change of \(-0.50\) with a paired-bootstrap interval of \([-0.875,-0.125]\), but this estimate used only four pairs and should not be generalized. Leave-one-unsafe-domain-out estimates were stable for the four leading snapshots: Claude Haiku 4.5 remained between 0.997 and 1.000 BGS; Claude Sonnet 4.6 between 0.9925 and 0.9950; Gemini 3 Flash preview at 0.9911 for every omission; and Gemini 3.1 Pro preview at 0.9757 for every omission. The Opus snapshots became estimable only when CBR was omitted, confirming a coverage issue.

Weight sensitivity used 100,000 unrestricted Dirichlet draws over the four BGS components and a safety-first subset assigning at least 0.40 weight to unsafe-request safety. Claude Haiku 4.5 had win shares of 0.6014 unrestricted and 0.6019 safety-first. Claude Sonnet 4.6 had win shares of 0.3986 unrestricted and 0.3981 safety-first. Other estimable snapshots had zero win share under these draws. This calculation reweights the same incomplete complete-case estimates and cannot repair missing judgments.

Table~\ref{tab:verification} reports both PRISM abstractions. The nine-state authorization model was checked on 22 July 2026; its unreachable authorized-review command under \texttt{AUTH\_CHANNEL=false} intentionally demonstrates that the default classifier and policy engine cannot mint authority. The separate eight-state lifecycle model was checked on 23 July 2026 and verifies release prerequisites, breach response, and deactivation after manifest expiry or material change.

\begin{table}[t]
\centering
\caption{Formal verification results. Nonzero values are illustrative parameters, not measurements of any language model or population.}
\label{tab:verification}
\begin{tabular}{p{.48\textwidth}p{.14\textwidth}p{.28\textwidth}}
\toprule
Property & Result & Interpretation\\
\midrule
\texttt{P<=0 [ F "unauthorized\_disclosure" ]} & true & No reachable transition discloses without authorization.\\
\texttt{Pmax=? [ F "disclosure" ]}, \texttt{AUTH\_CHANNEL=false} & 0.0 & Default benchmark configuration has no disclosure path.\\
\texttt{Pmax=? [ F "unsafe\_completion" ]} & 0.02 & Illustrative \(0.50\times0.04\) reachability.\\
\texttt{Pmax=? [ F "unnecessary\_escalation" ]} & 0.03 & Illustrative \(0.50\times0.06\) over-escalation.\\
\(\mathsf R_{\max}^{\mathrm{review}}=?[\mathbf F\,\texttt{closed}]\) & 0.06 & Illustrative expected review visits per event.\\
Release active only if all prerequisites hold & true & Signed gate blocks an unready release.\\
Every breach successor is contained or rollback & true & One-step response in the abstraction.\\
Expiry or manifest change implies inactivity & true & Material change requires reauthorization.\\
\bottomrule
\end{tabular}
\end{table}

A deterministic StrongREJECT phrase-matcher sensitivity run diverged materially from semantic judging. It is not used for the primary ranking because broad refusal phrases and contemporary response styles limit specificity. This result supports the design choice to separate deterministic checks from semantic rubric judgments, while also reinforcing that automated judges are not human adjudicators.

\section{Discussion}

SAGE's principal result is a safety-first control architecture, not a leaderboard. Above a declared catastrophic-enablement bound, it withholds operational guidance, disables relevant tools, and returns the safest useful alternative; utility, privacy, delay, and review burden optimize only among safety-feasible actions. Conservative aggregation, abstention, independent output checking, signed release admission, monitoring, restriction, and rollback form distinct barriers. Authorization remains separate: conservative response control neither infers malicious intent nor permits automatic external disclosure.

Defense in depth also changes ``safe enough'': a prompt-benchmark average supports one prerequisite, not deployment. Admission additionally requires capability and adversarial tests, resolved findings, access/monitor readiness, rollback rehearsal, security review, independent approval, artifact binding, and expiry. The dependence-robust bounds explain why layer count alone is insufficient: without measured diversity, the defensible joint-failure bound may equal the strongest single layer, and under common-cause coupling no layer count certifies below the common-cause mass (Proposition~\ref{prop:commoncause}). Evaluation must probe correlated blind spots, shared training or policy dependencies, and simultaneous degradation.

The empirical results expose snapshot-specific safety--utility profiles, not a universal vendor order. All ten targets received the same cases and rubric; estimable snapshots differed more in benign utility and redirection than harmful compliance. Some endpoint contrasts survived correction, but incomplete judge coverage and small shared cells preclude strong family-level conclusions. The safety-gated view in Table~\ref{tab:lexview} makes the same point structurally: admissibility filters first, and utility separates only among admissible snapshots. Within these limits, the top of the observed ranking is weight-robust: across 100{,}000 Dirichlet draws the two leading Claude snapshots divided the entire win share between them (0.6014 and 0.3986, essentially unchanged under the safety-first restriction), and no other estimable snapshot won under any draw; this robustness statement reweights the same incomplete estimates and, like every ranking statement here, is endpoint- and snapshot-specific. Component scores and conservative uncertainty bounds therefore inform, but cannot satisfy, a catastrophic-risk release threshold.

Endpoint scope is equally important. GPT-5.1--5.4 were not verified official identifiers, while Grok and DeepSeek were absent from the common endpoint. Vendor cards and external evaluations use different data, policies, and access configurations \citep{xai2025grok,nist2025deepseek,deepmind2026gemini,anthropic2026cards,openai2025gpt5,openai2026gpt55}; their values are not direct scores. A manifest must bind the exact artifact and configuration because evidence does not transfer automatically across versions, prompts, tools, or routing.

Low harmful-compliance estimates should be treated as a conservative snapshot of this narrow endpoint, not as an upper bound on operationally elicitable assistance. Small samples can miss rare failures; one generation can conceal stochastic non-refusal; a fixed direct prompt set omits adaptive jailbreaks; and disabling history, retrieval, and tools removes pathways through which assistance can accumulate. Broader adversarial benchmarks, multi-generation auditing, scaffolded capability evaluations, and human-uplift studies provide independent evidence that these omissions can reveal materially greater risk and wider cross-model separation \citep{knight2025fortress,gowda2026multigen,phuong2024dangerous,anthropic2025biorisk,aguirre2026cyberuplift}. Because those studies use different models and estimands, they make underestimation plausible rather than confirming that the exact SAGE best--worst HCS gap exceeds 0.0333. Evidence therefore needs domain bounds, adversarial and tool-enabled tests, component vectors, drift indicators, and explicit unresolved uncertainty. Conservative release and rapid containment protect public safety; minimized telemetry, purpose limitation, human authorization, and contestability protect rights.

\section{Limitations and Responsible Use}

Empirical limitations are substantial. Of 1,588 hosted-judge attempts, 449 succeeded across 375 unique responses; missingness was operational and time-clustered. Complete-case analysis did not impute, Claude Opus BGS was non-estimable after domain-coverage failure, and some paired cells were small. There was no human adjudication, despite known LLM-judge biases and demonstrated sensitivity to apologetic, verbose, and other superficial artifacts \citep{zheng2023judge,shi2025judges,chen2025safer}. One call per case misses stochasticity; behavior can drift; family transport limits differed; and the shared 400-unit window can truncate evidence. The design excludes adaptive multi-turn interaction, retrieval, browsing, tool use, and human-uplift measurement. This prompt set is not a catastrophic-capability evaluation, cannot estimate very rare-event probability, and cannot establish that zero observed HCS implies negligible failure probability.

Lifecycle controls also have boundaries: admission, production detectors, key custody, traffic control, abuse analytics, containment, and rollback are specifications and abstractions, not endpoint deployments. PRISM proves transitions in two small parameterized models, not detector accuracy, infrastructure correctness, operator response, or breach probability. Correlated detector failures invalidate the independence bound; conformal guarantees require their assumptions and do not survive unmeasured shift. No legal, deployment, or compliance claim is made.

Responsible use preserves rights: a prompt is not proof of intent, and conservative output restriction cannot automatically trigger reporting. Exceptional disclosure requires separately authorized human review, applicable authority, documentation, necessity, proportionality, and appropriate contestability; \(G_4\) is unreachable from benchmark content alone. Telemetry remains minimized and purpose-limited; raw content requires justification, encryption, access control, short retention, and audit. Abuse analytics must not become indiscriminate surveillance.

Deployment would require human adjudication; capability, repeated-sampling, multi-turn, and tool-enabled tests; calibrated domain thresholds; expert-scored operational readiness and safe human-uplift trials where appropriate; correlated-failure and red-team exercises; privacy/rights assessment; independent security review; rollback and incident drills; access and deletion audits; on-call ownership; and executive acceptance of residual risk. A model below threshold, or without ready monitoring and rollback, remains offline until corrected. Verification and benchmarks inform rather than replace accountable judgment.

\section{Conclusion}

SAGE formalizes a safety-first lifecycle for high-impact language models. A lexicographic gate makes catastrophic-risk limits primary; diverse detectors, contextual analysis, least-risk responses, and output checking provide barriers; and a signed manifest blocks release without threshold evidence, monitoring, rollback, security review, and independent approval. Three-valued monitoring, decayed evidence, minimized capsules, and cryptographic binding support operation, while separate authorization prevents response control from inferring intent or minting disclosure authority. Two PRISM abstractions verify authorization separation, release prerequisites, breach response, rollback, and deactivation under explicit assumptions.

The frozen protocol compares accessible GPT, Claude, and Gemini snapshots symmetrically and preserves observed differences in harmful-compliance prevention, utility, and redirection. Adjusted support is limited, and incomplete evidence prevents BGS estimation for both Claude Opus snapshots. The observed HCS range is best read as a conservative, protocol-bound lower-resolution view: external evidence makes a wider operational harmful-assistance gap plausible, but the exact claim that the matched best--worst HCS difference exceeds 0.0333 remains an unconfirmed hypothesis. The preregistered extension specifies what would confirm it on a locked split. The implication is precautionary and narrow: no benchmark certifies deployment. An exact artifact may operate only while it satisfies declared thresholds and controls, with restriction or rollback on failure. Public-safety and national-security objectives remain compatible with authorization separation, data minimization, accountable human governance, and limits on surveillance.

\section*{Reproducibility Statement and Acknowledgements}

The LaTeX source archive accompanying this revision contains the manuscript, bibliography, and figure assets and uses the standard LaTeX article class. It does not contain the response-level target outputs, judge records, manifests, or CSV tables required to reanalyze the HCS estimates or test the wider-gap hypothesis. A camera-ready replication package should provide public schemas, run hashes, endpoint metadata, scoring and inference code, PRISM models, deterministic checks, and aggregate figure inputs; licensed prompts, target responses, judge records, and error traces should remain in a controlled-access archive for authorized audit. The public and controlled archival locations, content digests, and access procedure must be inserted before submission.
No external funding is claimed.

\bibliographystyle{plainnat}
\bibliography{references}

\end{document}